\documentclass[twocolumn]{ustc_conference}
\usepackage{amssymb}
\usepackage[utf8]{inputenc}
\usepackage{booktabs}
\usepackage{multirow}
\usepackage{xcolor}
\usepackage{colortbl}
\usepackage{adjustbox}
\usepackage{wrapfig}
\usepackage{needspace}
\usepackage[edges]{forest}
\usepackage{tikz}
\usepackage{amsfonts}
\usepackage[toc,page,header]{appendix}
\usepackage{bbm}
\usepackage{enumitem}
\usepackage{amsmath}
\usepackage{floatflt}
\usepackage{bm}
\usepackage{graphicx}
\usepackage{colortbl} 
\usepackage{enumitem} 
\usepackage{wrapfig}
\usepackage{listings} 
\usepackage{tcolorbox}
\usepackage{algorithm}
\usepackage{algorithmic}  
\usepackage{multirow} 
\usepackage{pifont} 
\usepackage{newfloat}
\usepackage{pifont}
\usepackage{arydshln}   
\usepackage[dvipsnames,table]{xcolor}
\usepackage{mathtools}
\newtheorem{theorem}{Theorem}[section]

\newtheorem{proof}[theorem]{Proof}

\usepackage[textsize=tiny]{todonotes}
\usepackage[most]{tcolorbox}
\usepackage{makecell}

\definecolor{frameorange}{RGB}{218, 114, 27}
\definecolor{bgorange}{RGB}{253, 243, 235}

\definecolor{frameblue}{RGB}{0, 85, 150}
\definecolor{bgblue}{RGB}{235, 245, 255}

\definecolor{framegreen}{RGB}{34, 139, 34}
\definecolor{bggreen}{RGB}{240, 253, 240}

\newtcolorbox{StrategyBox}[3][frameorange]{
  enhanced,
  float*,
  width=\textwidth,
  title={#3},
  colframe=#1,
  colback=#2,
  colbacktitle=#1,
  coltitle=white,
  fonttitle=\bfseries\large,
  fontupper=\rmfamily,
  arc=1.5mm,
  boxrule=1.2pt,
  top=3mm, bottom=3mm, left=3mm, right=3mm,
  toptitle=0.5mm, bottomtitle=0.5mm,
  before upper={\setlength{\parindent}{1.5em}}
}

\newtcolorbox{BreakableStrategyBox}[3][frameorange]{
  enhanced,
  breakable,
  width=\linewidth,
  title={#3},
  colframe=#1,
  colback=#2,
  colbacktitle=#1,
  coltitle=white,
  fonttitle=\bfseries\large,
  fontupper=\rmfamily,
  arc=1.5mm,
  boxrule=1.2pt,
  top=3mm, bottom=3mm, left=3mm, right=3mm,
  toptitle=0.5mm, bottomtitle=0.5mm,
  before upper={\setlength{\parindent}{1.5em}}
}

\renewcommand{\paragraph}[1]{\vspace{0.1em}\noindent\textbf{#1}}

\title{DiDPO: Diff-in-Diff Policy Optimization for \\ Coding Agent Training}

\author{%
\parbox{\textwidth}{\centering
Xucong Wang\frontsup{1},  Zhe Zhao\frontsup{1,2}, Liheng Yu\frontsup{1}, Di Wu\frontsup{1}, \\ Xiaofeng Cao\frontsup{4}, Pengkun Wang\frontsup{1,3\corrauthor} 
}}

\affiliation{%
\parbox{\textwidth}{\centering
\affilsup{1}University of Science and Technology of China  (USTC)  \quad \affilsup{2}Stanford University   \\
\affilsup{3}Suzhou Institute for Advanced Research, USTC \quad 
\affilsup{4}Tongji University
}}

\contribution[\mathsection]{Corresponding Author}

\abstract{
Reinforcement learning with Verifiable Reward (RLVR) has emerged as a powerful paradigm for training coding agents, where the execution feedback from compilation and tests provides objective verification. However, unlike agent tasks, coding agents face a unique and finer-grained credit assignment challenge: at each step, coding actions simultaneously pack varying changes into different regions of a code version, which makes the contribution of independent change indistinguishable. Existing RLVR methods mostly leverage the outcome reward or step-level reward, which fails to dive into a code diff and makes unique properties of coding actions invisible to training. In this paper, we propose Diff-in-Diff Policy Optimization (DiDPO), a critic-free RL method that constructs fine-grained credit units directly from the structure of code diffs. DiDPO organizes multi-turn coding interactions into multiple thought--action steps and discovers code diffs across sampled trajectories. It then selects anchors by aggregating highly similar sub-diffs split from each whole diff by our ``groupability score'', which provides the splitting schema that optimally balances the semantic scope of anchors and the group mass they may form. Finally these anchors form advantage groups and project the diff-level advantage back to individual response tokens. Experiments on long-horizon coding and reasoning benchmarks show that DiDPO significantly outperforms strong agentic RL baselines. On Qwen2.5-7B-Coder, DiDPO exceeds comparable methods by over 10\% and narrows the gap with far larger models, offering a principled framework for fine-grained credit assignment in coding agent training. We also open-source  verl-code, an agentic rl codebase that supports various RL methods and coding benchmarks.

}

\checkdata[Github]{\url{https://github.com/xuc865/verl-code}}

\begin{document}

\maketitle


\section{Introduction}
LLM agents~\cite{yao2023react,shinn2023reflexion,chhikara2025mem0,fang2026memp,lee2026meta,yao2022webshop} extend language models from response generation to goal-directed interaction by combining reasoning, tools, and environmental feedback. ReAct~\cite{yao2023react} established a widely used reasoning--action loop, while related self-improvement methods showed how environmental feedback can refine subsequent behavior~\cite{shinn2023reflexion}. For coding, advances in code representation and generation provided the foundation for understanding program structure and producing executable code~\cite{fried2023incoder,li2023starcoder,hui2024qwen25coder}.  Some coding agents, like SWE-agent~\cite{yang2024swe}, OpenHands~\cite{wang2024openhands}, CodeAct~\cite{wang2024codeact} also place the model inside an interactive software workspace, allowing it to inspect repositories, modify files, execute code, and respond to test feedback. Recent work also explores stateful agent harnesses and streamlined repository-repair workflows~\cite{xia2024agentless}. Alongside this methodological progression, coding benchmarks have expanded from executable synthesis to competition-level programming and real repository repair tasks~\cite{chen2021evaluating,jimenez2024swebench,ma2026skillclaw}.

\begin{figure}[t] 
    \centering
    \includegraphics[width=1\linewidth]{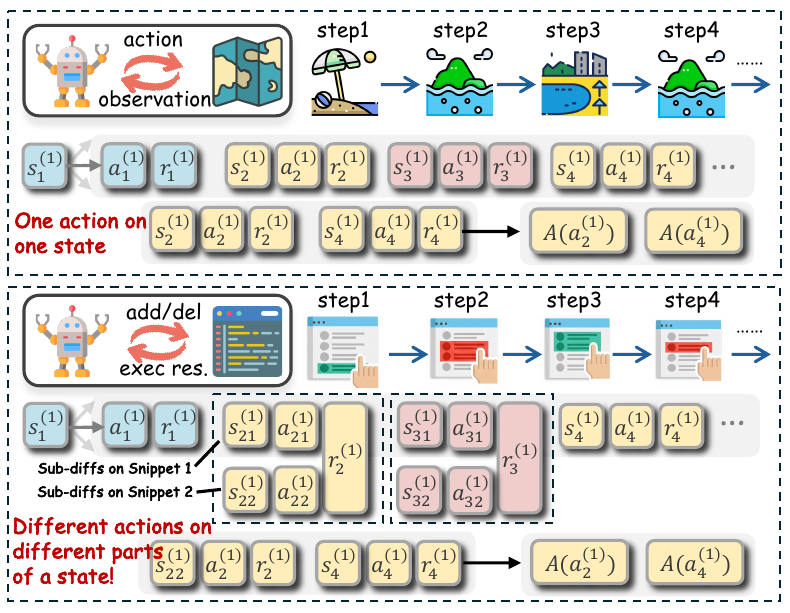} 
    \vspace{-0.7cm}
    \caption{\textbf{Upper:} For many agent tasks, each agent action is treated as a single decision unit over a state. \textbf{Lower:} For coding tasks, many coding actions may be applied to different parts of a state, while producing multiple functional sub-diffs.} 
    \vspace{-0.5cm}
    \label{fig:concept}
\end{figure}

These executable environments create a natural role for Reinforcement Learning with Verifiable Rewards (RLVR). A software issue may admit many valid interaction trajectories, making exhaustive supervision expensive, while compilation and tests provide objective feedback on the resulting workspace. CodeRL~\cite{le2022coderl} first used execution outcomes to optimize program generation, and SWE-RL~\cite{wei2025swerl} later extended reinforcement learning to open software evolution. ExecVerify~\cite{tang2026execverify} further shows the value of stepwise verifiable signals for code execution reasoning. More broadly, GRPO~\cite{guo2025deepseekr1} estimates relative advantages from multiple rollouts without a learned critic, while DAPO~\cite{yu2025dapo} and GSPO~\cite{zheng2025gspo} improve the stability and efficiency of group-based policy optimization. For long-horizon agents, GiGPO~\cite{feng2025gigpo} groups actions that revisit the same environment state. Together, these methods make RLVR increasingly effective for training interactive agents.

However, as shown in Figure \ref{fig:concept}, the common settings in many agent tasks~\cite{shridhar2021alfworld,wang2023cogagent,wang2026role} assume one atomic action is applied to one atomic state, which does not transfer cleanly to coding. To be concrete, \ding{182} coding actions are bundled with snippets which expand with task difficulty and project scale. This means  harder tasks require longer edits and admit interleaved actions on interleaved regions. \ding{183} At each step, each diff (\textbf{i.e., code changes between two adjacent versions}) may contain several sub-diffs (\textbf{i.e., part of the code diffs}) with different functionalities. Treating the whole diff as one therefore obscures the semantic unit that produced an improvement. \ding{184} code is also governed by strict syntax, execution, and pass/fail criteria; any small malformed edit may influence a long trajectory thus merits special focus.  These challenges naturally lead to the following question:

\begin{tcolorbox}[
  colback=gray!10,     
  colframe=black,     
  boxrule=1pt,         
  arc=2mm,            
  left=4pt,right=4pt,top=4pt,bottom=4pt,  
]
 \textbf{How can we construct dynamic, finer-grained credit units for code diffs to better apply agentic RL to coding tasks?}
\end{tcolorbox}

Motivated by these observations, we introduce Diff-in-Diff Policy Optimization (DiDPO), a tailored agentic RL algorithm that uses the hierarchy in each code diff to construct fine-grained credit. DiDPO retains a trajectory-level group-relative advantage but further looks inside each code-producing action. Across rollouts of the same task, it identifies highly similar sub-diffs between diffs and uses these sub-diffs as anchors that induce sub-diff boundaries. The resulting sub-diffs capture recurring functional regions from diffs of different rollouts. DiDPO then groups aligned sub-diffs and computes a local group-relative advantage named diff-level advantage. Since over-large and over-small anchors affect the quality of grouping, we design a groupability score to select anchors that optimally balance semantic scope they could represent and the mass of the group they could aggregate. The diff-level advantage is combined with the trajectory-level advantage and projected back to its response tokens, requiring neither an additional critic nor extra environment rollouts. Extensive experiments across long-horizon coding benchmarks demonstrate the superiority of DiDPO over existing coding and agentic RL baselines.  Our contributions are threefold:
\begin{itemize}
  \item We identify three structural properties that distinguish coding from existing agentic RL settings and formulate the resulting need for hierarchical credit assignment within code-producing actions.
  \item We propose DiDPO, the first agentic RL algorithm tailored to the functional hierarchy of code, which dynamically aligns diffs, constructs sub-diff groups, and assigns group-relative advantages to the corresponding response tokens without extra rollouts.
  \item We conduct extensive experiments across diverse long-horizon coding benchmarks, where DiDPO consistently outperforms strong policy-optimization baselines.
\end{itemize}

\section{Related Work}
\subsection{Agentic Reinforcement Learning}
Reinforcement Learning with Verifiable Reward~\cite{azar2024ipo} extends policy optimization to dynamic, open-ended tasks whose outcomes can be checked automatically. Early RLHF~\cite{christiano2017deep,ouyang2022training} methods learned reward models from human comparisons and optimized policies against those models. Direct preference methods later converted pairwise preferences into supervised objectives, with DPO providing a representative formulation~\cite{rafailov2023direct}.  More recent methods such as GRPO~\cite{guo2025deepseekr1}, GSPO~\cite{zheng2025gspo}, and DAPO~\cite{yu2025dapo} estimate advantages from multiple rollouts of the same prompt. Subsequent studies examine how sampling, clipping~\cite{chen2025minimax} and normalization further affects the training stability.

In complementary to the outcome rewards, Process-supervision methods provide finer feedback by assigning signals to intermediate reasoning steps~\cite{cobbe2021training,lightman2023lets,wang2023lightweight,wang2025multi}. Interactive search-augmented systems combine environment feedback with search or value-guided trajectory improvement. RAGEN~\cite{wang2025ragen} studies training stability when trajectories contain many dependent decisions~\cite{chen2025reinforcement}. State-based methods approach credit assignment through local comparisons. GiGPO constructs groups when trajectories revisit the same state~\cite{feng2025gigpo}, whereas GAGPO derives a TD/GAE-style advantage from estimated state values~\cite{zhu2026gagpo}. In contrast, DiDPO treats code changes inside diffs as the comparison units. This reflects the internal structure of coding actions beyond environment observations.

\subsection{Code Generation and Coding Agents}
Executable feedback made functional correctness directly optimizable~\cite{chen2021evaluating,apps2021hendrycks} beyond surface similarity rubrics. Building on this signal, CodeRL used unit tests to train program generators, while CodeT~\cite{chen2022codet} used generated tests and learned verifiers to improve candidate selection. Self-debugging~\cite{chen2024teaching} further incorporated execution failures into iterative revision. As executable evaluation expanded to data-science and class-level generation~\cite{lai2023ds1000}, repository-level tasks extended the same feedback loop across a sequence of workspace interactions~\cite{jimenez2024swebench}. SWE-agent realizes this process through repeated editing and testing~\cite{yang2024swe}, with format disciplines from OpenHands~\cite{wang2024openhands} and CodeAct~\cite{wang2024codeact}. Although this line of work improves generation under executable feedback, the final program  remains the unit of evaluation. As a result, the contribution of internal changes remains unresolved. 

In complement to the above, some methods target repair-generation and unit editing~\cite{just2014defects4j}.  GenProg~\cite{leGoues2012genprog} searches over test-guided patch candidates, while Prophet~\cite{long2016prophet} uses human-designed or learned priors to rank candidate repairs.   CURE~\cite{jiang2021cure} guides repair generation with code context, and Recoder~\cite{zhu2021recoder} incorporates syntax-guided decoding and neural patch generation. These methods, however, operate on isolated bug-fix instances where the edit scope is small and self-contained, whereas coding agents must distribute credit across multi-step trajectories that span multiple files and edit locations. In contrast, DiDPO is the first coding RL method to connect executable outcome feedback with localized credit by aligning recurring sub-diffs across coding agent rollouts.

\section{Problem Setting}
\subsection{Coding as Markov Decision Process}
We consider a typical multi-turn coding process interacting with a software environment. A task instance $x$ is provided to the agent. At the step $t$, the generation follows the CodeAct~\cite{wang2024codeact} paradigm, where the agent with the old policy $\pi_{\theta_{\mathrm{old}}}$ outputs $N$ Thought-Action-Observation cycles to the environments:
\begin{equation}
    \bm{\tau}^{(i)}=\{(\bm{s}^{(i)}_t,\bm{a}^{(i)}_t,r^{(i)}_t)|1\le t\le T_i\}
\end{equation}
Where $T_i$ is the number of steps in the trajectory $\bm{\tau}^{(i)}$, and the agent emits a textual action $\bm{a}^{(i)}_t$ based on the code state $\bm{s}^{(i)}_t$, and $\bm{a}^{(i)}_t$ can be \texttt{add}, \texttt{del} or \texttt{none}. \textbf{In contrast with many agentic tasks where each action is applied to one state, $\bm{s}^{(i)}_t$ here is divisible, and $\bm{a}_t$ here can be applied to a sub-set of $\bm{s}^{(i)}_t$.} For most coding tasks, the process reward $r$ would be replaced by the outcome reward $R^{(i)}$ for the $i$-th trajectory.

\begin{figure}[t] 
    \centering
    \includegraphics[width=1\linewidth]{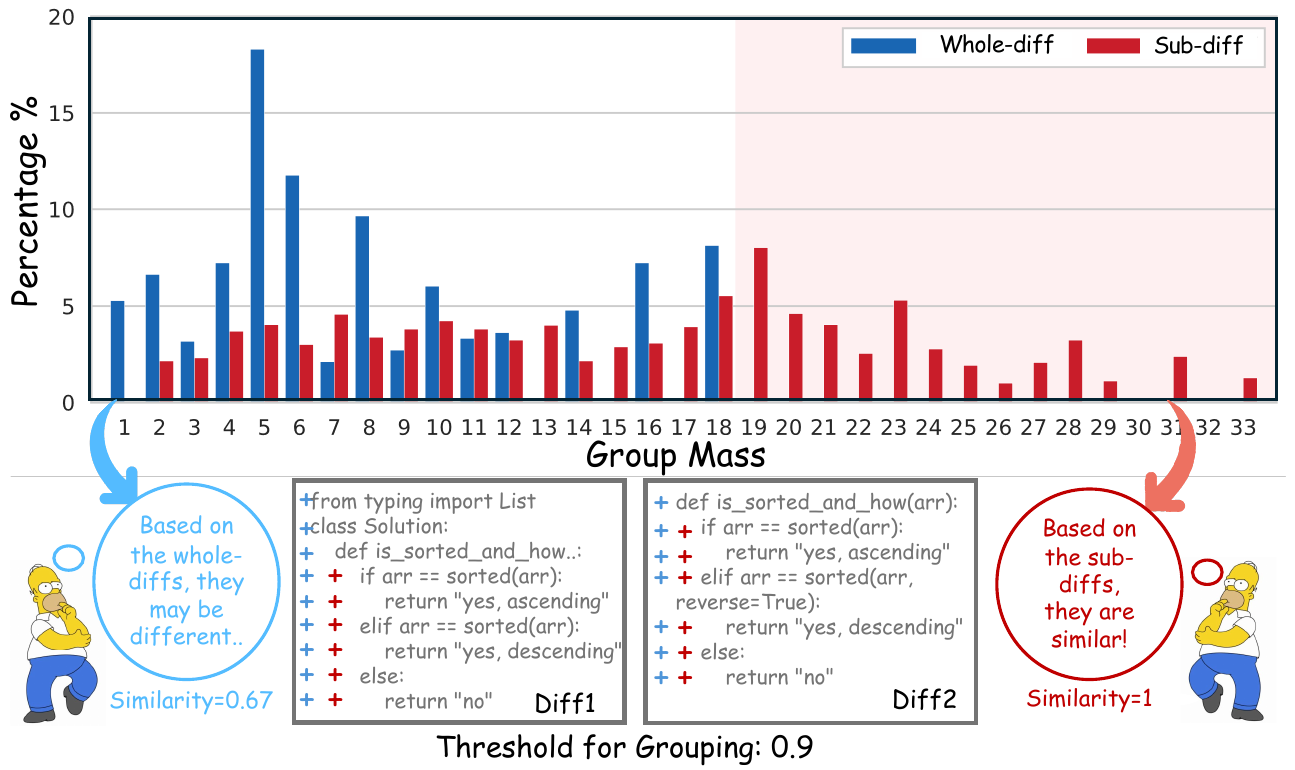} 
    \vspace{-0.7cm}
    \caption{\textbf{Upper:} Grouping states based on whole diffs tends to produce more small-mass \textbf{(mass: number of elements in that group)} groups; grouping based on sub-diffs results in more large-mass groups. This is because the former one neglects similarities in sub-diffs \textbf{(Lower)}.} 
    \vspace{-0.5cm}
    \label{fig:show}
\end{figure} 

\subsection{Trajectory / State-Level Advantage}
 For a group of rollouts sampled from the task prompt $x$, the trajectory-level advantage~\cite{shao2024deepseekmath} is written as:
\begin{equation}
A^E(\bm{\tau}^{(i)}) = [R^{(i)} - \frac{1}{N}\sum_{j=1}^{N} R^{(j)}]/{\rm std}(\{R^{(j)}\}_{j=1}^N),
\end{equation}
While $A^E(\bm{\tau}^{(i)})$ still treats the response as a single carrier of credit, GiGPO~\cite{feng2025gigpo} proposes that states naturally recur across episodes and results in redundancy; based on which, it constructs a state-level advantage term by grouping similar actions with similar states as units for credit assignment. Formally, GiGPO identifies unique states $\bm{s}$ from all trajectories of the same prompt as anchors:
\begin{equation}
    G(\tilde{\bm{s}})\!=\!\{(\bm{a}^{(i)}_t,r^{(i)}_t)|\bm{s}^{(i)}_t=\tilde{\bm{s}},1\!\le\! i \!\le\! N,1\!\le\!t \!\le\!T_i \}
\end{equation}
Let $R^{(i)}_{t}\!=\!\sum_{k=t}^{T_i}(\gamma)^{T_i-k}\!\cdot  r_k^{(i)}$ as the discounted reward at step $t$. The state-level advantage $A^S(\bm{\tau}^{(i)})$ is then calculated within actions started from each unique $\bm{s}$:
\begin{equation}
    A_i^S(\bm{a}^{(i)}_t)\!=\!\frac{R^{(i)}_{t}\!\!-\!\!{\rm avg}(\{R_{t}^{(j)}|(\bm{a}_t^{(j)},R_{t}^{(j)})\!\in\! G(\tilde{\bm{s}})\})}{F_{norm}(\{R_t^{(j)}|(\bm{a}_t^{(j)},R_t^{(j)})\!\in\! G(\tilde{\bm{s}})\})}
\end{equation}
$A_i^S$ is then combined with $A^E(\bm{\tau}^{(i)})$ to provide state-level signals for intermediate actions in long-horizon agent tasks.

\section{Methodology}  
The primary challenge in long-horizon coding tasks lies in identifying the contribution of each action step. To address this, we introduce DiDPO, which  derives dynamic advantage groups based on code \textbf{diffs} for credit assignment.

\begin{figure*}[t] 
    \centering
    \includegraphics[width=1\linewidth]{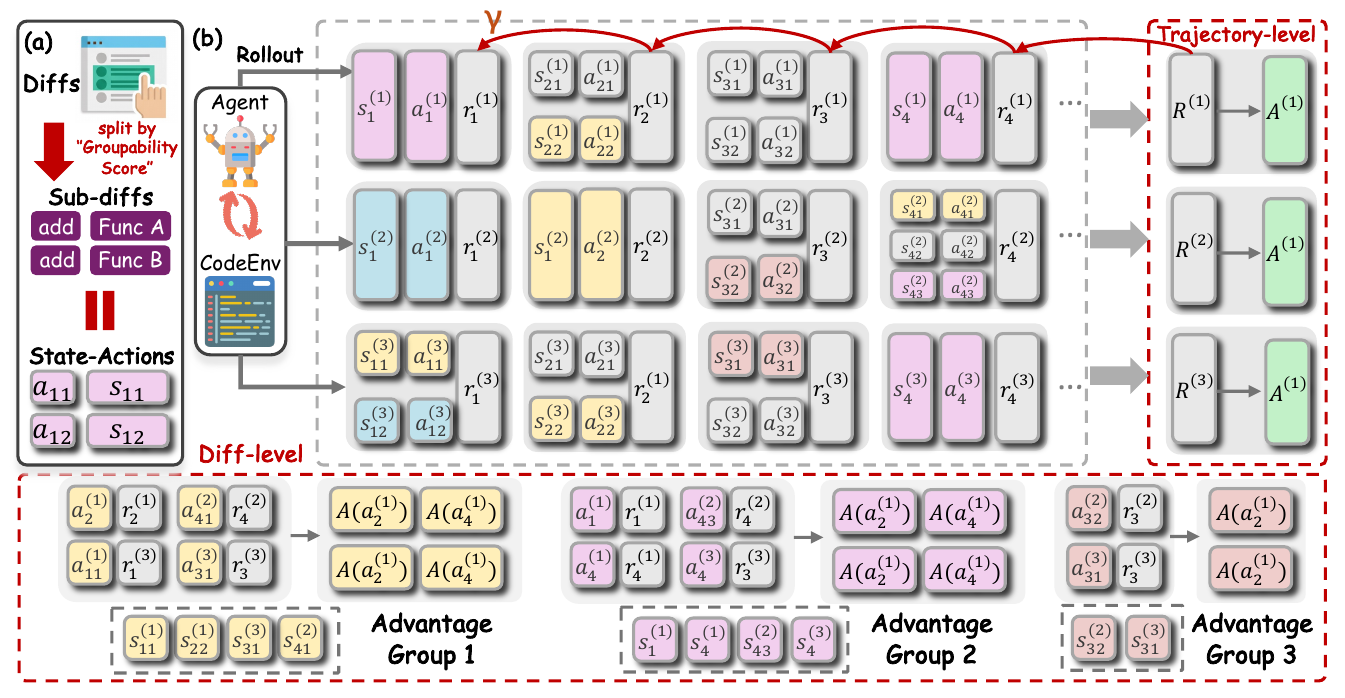} 
    \vspace{-0.7cm}
    \caption{\textbf{(a)}: Code diffs can be split into sub-diffs according to our proposed Groupability Score (GS). \textbf{(b):} The agent interacts with the code execution environments and generate several code editing trajectories, and state-actions with the same color means the same sub-diffs. DiDPO computes both trajectory-level advantage $A^E$ and diff-level advantage $A^D$; $A^E$ is calculated from outcome rewards, while $A^D$ is calculated among all aggregated sub-diffs. } 
    \vspace{-0.5cm}
    \label{fig:main}
\end{figure*}

\subsection{Diffs are Divisible}
For step $k$ of trajectory $\bm{\tau}^{(i)}$, let $\mathcal{D}_{i,k}$ denote the set of diffs distributed at different locations. DiDPO first aggregate diffs from all trajectories and steps together, resulting $\mathcal{S}$: 
\begin{equation}
\mathcal{S}= \{\{\mathcal{D}_{i,k};1\le k\le T_i\};1\le i \le N \}
\label{eq:snippet-set}
\end{equation}
To better track these diffs for differentiated processing, for diff each  $\bm{s}\in\mathcal{S}$, we attach the following tuple of metadata to this diff:
\begin{equation}
   {\rm Metadata}(\bm{s}) = (u(\bm{s}),v(\bm{s}),w(\bm{s}),q(\bm{s})),
\end{equation}
where $u(\bm{s})$ denotes the \textbf{normalized} (i.e., with annotations and blanks removed) text of diff $\bm{s}$, $v(\bm{s})$ denotes its response-token span, $w(\bm{s})\in\{\texttt{add}, \texttt{del},  \texttt{none}\}$ represents the actions applied to the diff, and $q(\bm{s})$ the task instance that produces $\bm{s}$. The normalized diff size is denoted  as $|u(\bm{s})|$.

Unlike cases  where states can be treated as atomic units, when code diffs are placed together, the severe asymmetry prevents most diffs from being effectively grouped. As shown in Figure~\ref{fig:show}, on the APPS dataset, grouping states based on the whole diffs results in groups that are mostly small-sized, whereas decomposing diffs into finer sub-diffs leads to a greater number of groups with larger sizes. A typical example to explain this is: Diff 1 and Diff 2 partially overlap but also partially diverge, causing the overall similarity score to fall below the similarity threshold. This motivates us that further division of diffs helps uncover a richer set of mergeable states.

\subsection{Dynamic Sub-Diff Anchors}
The above findings prompted us to further decompose diffs, yet this raises another question: \textbf{how should we choose the division criteria for different diffs, to ensure that the resulting sub-diffs are the most abundant and semantically meaningful?} If we uniformly split all diffs  line by line and rely on them to form anchors, the anchors would inevitably be cluttered with countless meaningless blanks and snippets. DiDPO addresses this by striking a balance between the size of the anchors and the group mass \textbf{(number of elements in that group)} of anchors.

Recall that $|u(\bm{s})|$ is the normalized diff size of $\bm{s}\in\mathcal{S}$. We enumerate its contiguous sub-diffs at multiple scales as the following:
\begin{equation}
\mathcal{I}(\bm{s})=\big\{{\rm Seg}(\bm{s},[p,q)):{}0\le p<q\le |u(\bm{s})|\big\},
\end{equation}
Where ${\rm Seg}(\bm{s},[p,q))$ slices the diff $\bm{s}$ from $p$ to $q$. For two sets of sub-diffs $\mathcal{I}(\bm{s})$ and $\mathcal{I}(\bm{s}')$, which may be produced by different rollouts or steps, we separately calculate the similarity matrix as
\begin{equation}
\begin{aligned}
\mathcal{M}_{\bm{s},\bm{s}'}=&\big\{(I,J,\sigma_{I,J})|
I\in\mathcal{I}(\bm{s}),\ J\in\mathcal{I}(\bm{s}');\\
w(\bm{s})=\, & \,w(\bm{s}');\sigma_{I,J}=\mathrm{sim}\!\left(u(\bm{s}_I),u(\bm{s}'_J)\right)
\ge\eta\big\}.
\end{aligned}
\label{eq:anchor-match}
\end{equation}
The ${\rm sim}$ combines token-level lexical  matching and embedding  similarity to compare the source text. Notably matching is restricted to the same edit type, i.e, $w(\bm{s})\!=\!w(\bm{s}')$.

All matched sub-diffs can be aggregated into an anchor $\bm{c}\in\mathcal{C}$, which represents one code-changing pattern. We use $\mathcal{O}(\bm{c})$ to denote the occurrences supporting this anchor. Accordingly, $\bar L(\bm{c})$ measures its average size and $n(\bm{c})$ counts the occurrences. An optimal anchor should both \textbf{be large enough to carry a valid functional meaning} and also \textbf{gather a group with enough group mass}. We define the ``Groupability Score (GS)'' to formalize this trade-off:
\begin{equation} 
\mathrm{GS}(\bm{c})=
\big(1-e^{-\bar L(\bm{c})}\big)
\big(1-e^{-(n(\bm{c})-1)}\big),
\label{eq:anchor-groupability}
\end{equation} 
While Equation~\ref{eq:anchor-groupability} shows how suitable each occurrence is as an anchor, two high-scoring anchors may be built from many of the same sub-diffs. Thus, selecting anchors by this score alone would therefore introduce redundant groups. 

With this in mind,  DiDPO searches for different anchors which jointly maximizes GS across sub-diffs:
\begin{equation} 
\mathcal{C}^{\star}
 \! = \!{\rm argmax}_{\substack{|\mathcal{C}|\le K}}  \!\!\!\!\sum_{s\in\mathcal{I}(\mathcal{S})}\!\!
\big\{\!\max_{\bm{c}\in\mathcal{C}}\mathbf{1}\!\left[s\!\!\in\!\!\mathcal{O}(\bm{c})\right]\mathrm{GS}(\bm{c})\big\}
\label{eq:anchor-selection}
\end{equation}
The inner maximum allows each candidate sub-diff to contribute through only one selected anchor, so overlapping anchors cannot repeatedly claim the same occurrence. Equation~\ref{eq:anchor-selection} is a cardinality-constrained facility-location form of submodular maximization~\cite{nemhauser1978analysis}. We solve it greedily by adding the anchor with the largest marginal increase until reaching no gains.  The detailed solution process is in  \textbf{\underline{Appendix ~\ref{appx:b}}}.

\subsection{Diff Advantage Calculation}
The selected anchors determine both the decomposition and the grouping. According to the anchors $\mathcal{C}^{\star}$, larger diffs are cut into the occurrences of anchors and other sub-diffs. Actions belonging to the same anchor then form an advantage group:
\begin{equation}
{G}(\mathcal{C}^{\star})\!=\!\{(\bm{a}^{(i)}_{t,m}  ,R^{(i)}_t)|\exists  \bm{c}\in\mathcal{C}^{\star}:\bm{s}^{(i)}_{t,m}\in\mathcal{O}(\bm{c}) \}.
\label{eq:anchor-assignment}
\end{equation}
Where $\!1\!\le\!\! i \!\le\! N,\!1\le\!\! t \!\le\!  T_i$.  $\bm{a}^{(i)}_{t,m}$ is the action over the $m$-th sub-diff.  $R^{(i)}_{t}$ is the discounted reward at step $t$. The groups ${G}(\mathcal{C}^{\star})$ therefore contain sub-diffs that share a code-changing pattern across rollouts. DiDPO then computes the diff-level advantage $A^{D}(\bm{a}^{(i)}_{t,m})$ for each $(\bm{a}^{(i)}_{t,m}\,,R^{(i)}_t)\in{G}(\mathcal{C}^{\star})$ as:
\begin{equation}
A^{D}\!(\bm{a}^{(i)}_{t,m})\!\!=\!\!
\frac{{R}_t^{(i)}\!\!-\!\!\mathrm{avg} (\{R_{t}^{(j)}|(\bm{a}_{t,m}^{(j)},R_{t}^{(j)})\!\in\!G(\mathcal{C}^{\star})\})}
{F_{norm}(\{R_{t}^{(j)}|(\bm{a}_{t,m}^{(j)},R_{t}^{(j)})\!\in\! G(\mathcal{C}^{\star})\})}. 
\label{eq:subdiff-adv}
\end{equation}
 The trajectory-level advantage $A^{E}$ supervises the entire coding, whereas the diff-level advantage $A^{D}$ distinguishes code diffs that lead to different outcomes within the same task.

\subsection{Diff-In-Diff Policy Optimization}
With a slight abuse of notation, let $\bm{a}^{(i)}_{l}$ be the action containing response $l$ from rollout $i$. Since each diff may contain multiple actions that belong to different advantage groups, we write the final advantage as a token-level format, with a coefficient $\lambda$:
\begin{equation}
\hat A_{i,l}=A^E(\bm{\tau}^{(i)})+\lambda\cdot A^{D}(\bm{a}^{(i)}_{l}).
\label{eq:DiDPO-adv}
\end{equation}
Finally, let $\rho_{i,l}(\theta)\!=\!\pi_{\theta}(y_{i,l}\!\mid\! y_{i,<l},x)/\pi_{\theta_{\mathrm{old}}}(y_{i,l}\!\mid\! y_{i,<l},x)$. The training objective of DiDPO applies the standard clipped form as:
\begin{equation}
\mathcal{J}\!\!=\!\mathbb{E}_{x,i,l}\!\!\left[
\min\!\left(
\!\!\rho_{i,l}\hat A_{i,l},
\mathrm{clip}(\rho_{i,l},1\!-\!\epsilon,1\!+\!\epsilon)\hat A_{i,l}
\!\right)\right] 
\label{eq:objective}
\end{equation}
All anchors and groups are obtained from the original rollouts, thus DiDPO  introduces minor additional interaction expenditure.

\subsection{Theoretical Foundations} 
We give a compact view of why DiDPO scores anchors by both semantic size and group mass. The size term avoids fragments that are too small to express a meaningful code change. The mass term controls the reliability of the local comparison once similar sub-diffs have been matched across rollouts. For a sub-diff $s$, let $U_s$ be its normalized code units and let $d_s$ encode their textual and structural relations. A correspondence $\pi\subseteq U_s\times U_{s'}$ aligns units from two sub-diffs. Following the Gromov-Hausdorff perspective~\cite{burago2001metric,memoli2011gromov}, its distortion is
\begin{equation}
\mathrm{dis}(\pi)\!\!=\!
\sup\nolimits_{{(a,b)\in\pi,(a',b')\in\pi}}
\left|d_s(a,a')\!-\!d_{s'}(b,b')\right| 
\end{equation}
We combine this structural distortion with the source similarity used by the anchor matcher:
\begin{equation}
\Delta(s,s')\!=\!\!\!\!\!\!
\inf_{\pi\in\Pi(s,s')}
\!\!\left[
\frac{1}{2}\mathrm{dis}(\pi)
\!+\!\beta\big(1-\mathrm{sim}_{\pi}(s,s')\big)
\!\right] 
\label{eq:local-correspondence}
\end{equation}
Where $\Pi(s,s')$ is the set of valid correspondences, $\mathrm{sim}_{\pi}$ measures aligned source similarity, and $\beta$ balances structural and textual agreement. This distance formalizes the mismatch between two sub-diffs that are grouped by an anchor.
\begin{theorem}
\label{thm:local-bias}
Suppose the local reward contribution $r(s)$ is $L$-Lipschitz with respect to $\Delta$. If every sub-diff in a DiDPO group is within correspondence error $\epsilon$ of the same anchor, then replacing exact code matches with anchor-based matches changes the local group contrast by at most $O(L\epsilon)$.
\end{theorem}
The first theorem states that better anchor alignment gives a less biased local credit signal. The next theorem will further show that grouping further reduces variance through cross-rollout averaging.
\begin{theorem}
\label{thm:lower-error}
Assume the step return decomposes as $R^{(i)}=r(s_i)+\xi_i$, where $s_i$ is the matched sub-diff and $\xi_i$ is zero-mean non-causal trajectory noise with variance $\sigma_\xi^2$. If a DiDPO group contains $m\!>\!1$  sub-diffs with correspondence error at most $\epsilon$, then its local advantage estimator satisfies
\begin{equation}
\mathrm{MSE}(\hat A_i^{D})
\le O(L^2\epsilon^2)+O(\sigma_\xi^2/m).
\end{equation}
An episode-level broadcast estimator retains an $O(\sigma_\xi^2)$ contamination term for tokens in the same sub-diff.
\end{theorem}
Together, the two results explain the statistical role of groupability. Anchor quality limits local-credit bias, while cross-rollout support reduces variance from non-causal trajectory components. In practice, $\epsilon$ is implicitly controlled by the similarity matching criteria and the groupability score used in anchor selection. Proof of theorems is in \textbf{\underline{Appendix ~\ref{appx:a}}}.

\begin{table*}[t]
\caption{Performance comparison on standard code generation benchmarks. The best results are in \textbf{bold}. }
\vspace{-0.4cm}
\centering
\label{tab:my-table:1}  
\setlength\tabcolsep{8pt} 
\resizebox{2\columnwidth}{!}{
\begin{tabular}{lccccccccl}
\hline
\multirow{2}{*}{\textbf{Method}} & \multicolumn{4}{c}{\textbf{APPS}} & \multirow{2}{*}{\textbf{HumanEval}} & \multirow{2}{*}{\textbf{MBPP}} & \multirow{2}{*}{\textbf{LiveCodeBench}} & \multirow{2}{*}{\textbf{Leetcode}} & \multicolumn{1}{c}{\multirow{2}{*}{\textbf{Avg}}} \\ \cline{2-5}
 & \textbf{All} & \textbf{Introductory} & \textbf{Interview} & \textbf{Competition} &  &  &  &  & \multicolumn{1}{c}{} \\ 
\textit{\# Samples} & \textit{5000}  & \textit{1000}  & \textit{3000}  & \textit{1000}  & \textit{164} & \textit{257} & \textit{1054}  & \textit{228} & — \\ \hline
Kimi-K2.6 & 84.4 & 92.5 & 80.1 & 89.0 & 97.0 & 82.2 & 74.9 & 72.4 & 82.2\\
Qwen3.6-27B & 87.6 & 91.6 & 87.7  & 83.4 & 96.9 & 79.1 & 72.2 & 62.3 & 79.6 \\
GLM-5.2 & 86.3 & 93.0 & 84.1 & 86.1 & 93.3 & 69.7 & 76.5 & 78.9 & 80.9 \\
GPT-5.5 & \textbf{89.5} & \textbf{95.5} & \textbf{88.6} & \textbf{86.4} & \textbf{100.0} & \textbf{83.7} & \textbf{100.0} & \textbf{85.2} & \textbf{91.7} \\
Skywork-OR1 & — & — & — & — & 87.2 & — & 33.8 & 60.0 & — \\ \hline  	 
Qwen2.5-Coder-7B & 16.9 & 40.8 & 12.7 & 5.5 & 69.3 & 61.1 & 15.0 & 14.3 & 35.3 \\  \hdashline 
\multicolumn{2}{l}{\textbf{\textit{Code-Reasoning Baselines}}} & \multicolumn{1}{l}{} & \multicolumn{1}{l}{} & \multicolumn{1}{l}{} & \multicolumn{1}{l}{} & \multicolumn{1}{l}{} & \multicolumn{1}{l}{} & \multicolumn{1}{l}{} &  \\
\! + \!CoT & 25.9 & 52.6 & 22.0 & 11.0 & 70.5 & 67.4 & 11.9 & 12.0 & 37.5 \\
\! + \!CodeAct & 14.7 & 34.4 & 11.6 & 4.4 & 66.0 & 62.0 & 19.1 & 18.2 & 36.0 \\
\! + \!Self-Planning & 17.4 & 40.4 & 13.6 & 6.0 & 68.9 & 66.0 & 26.5 & 15.5 & 38.9 \\
\hdashline 
\multicolumn{2}{l}{\textbf{\textit{Coding (Agentic) RL Baselines}}} & \multicolumn{1}{l}{} & \multicolumn{1}{l}{} & \multicolumn{1}{l}{} & \multicolumn{1}{l}{} & \multicolumn{1}{l}{} & \multicolumn{1}{l}{} & \multicolumn{1}{l}{} &  \\

\! + \!SFT & 16.0 & 38.3 & 11.0 & 8.8 & 66.2 & 64.2 & 19.1 & 21.5 & 37.4 \\
\! + \!CodeRL+ & 24.4 & 50.8 & 19.5 & 12.9 & 72.0 & 68.2 & 32.0 & 17.0 & 42.7 \\
\! + \!GRPO & 23.8 & 49.6 & 18.3 & 14.3 & 70.7 & 69.8 & 31.3 & 18.4 & 42.8 \\
\! + \!GiGPO & 25.1 & 53.8 & 18.0 & \textbf{17.7} & 67.0 & 71.7 & 35.0 & 22.1 & 44.2  \\
\rowcolor{cyan!20}
\! + \!DiDPO (ours) & \textbf{31.3} & \textbf{54.7} & \textbf{28.4} & 16.8 & \textbf{72.3} & \textbf{74.2} & \textbf{39.3} & \textbf{24.7} & \textbf{48.4}  \\ \hline
Qwen3.5-4B & 43.8 & 72.1 & 44.9 & 12.0 & 83.7 & 68.9 & 31.3 & 20.5 & 49.6 \\ \hdashline 
\multicolumn{10}{l}{\textbf{\textit{Code-Reasoning Baselines}}} \\
\! + \!CoT & 45.8 & 70.4 & 48.5 & 13.1 & 86.3 & 67.7 & 36.0 & 25.9 & 52.3 \\
\! + \!CodeAct & 40.7 & 70.5 & 39.1 & \textbf{16.0} & 84.2 & 61.5 & 33.0 & 25.7 & 49.0 \\
\! + \!Self-Planning & 45.1 & 74.0 & 46.4 & 12.0 & 85.4 & 54.9 & 26.7 & 28.1 & 48.1 \\ 
\hdashline 
\multicolumn{10}{l}{\textbf{\textit{Coding (Agentic) RL Baselines}}} \\ 

\! + \!SFT & 43.9 & 69.6 & 46.8 & 9.8 & 82.9 & 70.1 & 32.8 & 20.3 & 50.0 \\
\! + \!CodeRL+ & 47.1 & 74.1 & 50.1 & 11.0 & 85.0 & 70.7 & 35.3 & 27.7 & 53.2 \\
\! + \!GRPO & 45.0 & 76.0 & 46.2 & 10.4 & 85.2 & 67.4 & 30.2 & 25.4 & 50.6 \\
\! + \!GiGPO & 48.5 & 77.6 & 50.2 & 14.2 & 88.4 & 72.0 & 31.0 & 28.5 & 53.7 \\
\rowcolor{cyan!20}
\! + \!DiDPO (ours) & \textbf{51.3} & \textbf{82.2} & \textbf{53.0} & 15.5  & \textbf{91.8} & \textbf{76.9} & \textbf{40.4} & \textbf{32.8} & \textbf{58.6} \\
\hline
\end{tabular}}\vspace{-0.4cm}
\end{table*} 

\section{Experiments}

\subsection{Experimental Setup}
\paragraph{Benchmarks.} Following CodeRL+~\cite{jiang2026coderl+}, we use prime data from~\cite{cui2025process} as the  training set for all RL-based methods, but filter out those that belongs to the validation set or do not conform to long-horizon code generation patterns. For evaluation, eight code-generation benchmarks are selected: APPS~\cite{apps2021hendrycks}, HumanEval~\cite{chen2021evaluating}, MBPP~\cite{austin2021program}, LiveCodeBench~\cite{jain2025livecodebench}, LeetCode~\cite{xia2025leetcodedataset}, USACO~\cite{shi2024can}, OJBench~\cite{wang2025ojbench} and ICPC~\cite{xu2026icpc}. HumanEval and MBPP assess function-level synthesis, while APPS spans programming problems from introductory exercises to competition-level challenges. LiveCodeBench provides a continuously updated evaluation of code generation with reduced contamination risk. LeetCode, USACO, OJBench and ICPC place greater emphasis on longer competition-level algorithmic reasoning.

\paragraph{Training Pipeline.}
To standardize the boundaries between tool-call and reasoning in code generation, we place the target code in a sandbox environment (temp dir), where the agent can only perform an \texttt{add/delete/none} action at each step. 
For the training phase, since weaker models may not reliably follow the multi-turn thought-action pattern, we introduce a cold-start stage. 
Specifically, we first select a subset of medium-sized and moderate-difficulty tasks from the training set, and then augment them using template filling and rewriting with GPT-5.5. 
On the resulting 7K augmented dataset, we prompt Qwen3.6-27B to generate up to 12-turn rollouts while strictly adhering to the thought-action format. 
We then apply rejection sampling and LLM-based evaluation, ultimately collecting around 3K high-quality trajectories for Supervised Fine-Tuning (SFT). 
Finally, we apply DiDPO or other RL methods to the SFT checkpoint.

The training epochs for SFT and DiDPO stage are 2 and 120 respectively. Unless otherwise specified, all methods use the same set of hyper-parameters for fairness. Evaluation is performed with held-out tasks and deterministic test execution. More implementation details are in \textbf{\underline{Appendix ~\ref{appx:e}}}.

\textbf{Baselines.} We compare the proposed DiDPO with \textbf{(1):} Large open/closed-source models like Kimi-K2.6, Qwen3.6-27B, GLM-5.2, GPT-5.5 and Skywork-OR1~\cite{he2025skywork}; \textbf{(2):} Code-reasoning baselines, like CoT~\cite{wei2022chain}, CodeAct~\cite{wang2024codeact}, Self-Planning~\cite{jiang2024self}; \textbf{(3):} Coding RL / Agentic RL baselines, like CodeRL+~\cite{jiang2026coderl+}, GRPO~\cite{shao2024deepseekmath}, GiGPO~\cite{feng2025gigpo} and DiDPO (with SFT ablated). More introductions are in \textbf{\underline{Appendix~\ref{appx:c} / ~\ref{appx:d}}}.


\begin{table*}[t]
\caption{Performance comparison on competition-level algorithmic benchmarks. The best results are in \textbf{bold}. }
\vspace{-0.4cm}
\centering
\label{tab:my-table:2}  
\setlength\tabcolsep{11pt} 
\resizebox{2\columnwidth}{!}{
\begin{tabular}{lccccccccccc}
\hline
\multirow{2}{*}{\textbf{Method}} & \multicolumn{5}{c}{\textbf{USACO}} & \multicolumn{4}{c}{\textbf{OJBench}} & \multirow{2}{*}{\textbf{ICPC}} & \multirow{2}{*}{\textbf{Avg}} \\ \cline{2-10}
 & \textbf{All} & \textbf{Bronze} & \textbf{Silver} & \textbf{Gold} & \textbf{Platinum} & \textbf{All} & \textbf{Easy} & \textbf{Medium} & \textbf{Hard} &   &  \\
\textit{\# Samples} & \textit{307} & \textit{123} & \textit{100} & \textit{63} & \textit{21} & \textit{159} & \textit{20} & \textit{53} & \textit{86} & \textit{106} & — \\ \hline
Kimi-K2.6 & 79.1 & 95.0 & 79.8 & 64.4 & 16.7 & 32.3 & 89.5 & 54.3 & 5.5 & 5.7 & 39.0 \\
Qwen3.6-27B & 68.7 & 86.2 & 67.0 & 54.0 & 19.0 & \textbf{48.5} & \textbf{100.0} & \textbf{100.0} & 4.7 & \textbf{21.4} & 46.2 \\
GLM-5.2 &  63.5 & 80.5 & 64.0 & 49.2 & 4.8 & 13.8 & 70.0 & 15.1 & 0.0 & 5.7 & 27.7 \\
GPT-5.5 & \textbf{92.2} & \textbf{99.2} & \textbf{97.0} & \textbf{87.3} & \textbf{42.9} & 48.4 & 95.0 & 81.1 & \textbf{17.4} & 19.8 & \textbf{53.5} \\
Skywork-OR1 & — & — & — & — & — & — & — & — & — & — & — \\ \hline
Qwen2.5-Coder-7B & 6.5 & 9.8 & 5.0 & 4.8 & \textbf{0.0} & 3.8 & 25.0 & 1.9 & 0.0 & 0.9 & 3.7 \\ \hdashline   

\multicolumn{12}{l}{\textbf{\textit{Code-Reasoning Baselines}}} \\
\! + \!CoT & 12.7 & 24.4 & 7.0 & 3.2 & \textbf{0.0} & 8.6 & 50.0 & \textbf{3.8} & \textbf{1.9} & 7.5 & 9.6 \\
\! + \!CodeAct & 7.1 & 15.4 & 2.0 & 1.6 & \textbf{0.0} & 4.4 & 25.0 & \textbf{3.8} & 0.0 & 0.0 & 3.9 \\
\! + \!Self-Planning & 5.2 & 8.9 & 5.0 & 0.0 & \textbf{0.0} & 3.8 & 25.0 & 1.9 & 0.0 & 2.8 & 3.9 \\ \hdashline 
\multicolumn{12}{l}{\textbf{\textit{Coding (Agentic) RL Baselines}}} \\
\! + \!SFT & 5.2 & 10.6 & 2.0 & 1.6 & \textbf{0.0} & 3.8 & 25.0 & 1.9 & 0.0 & 4.7 & 4.6  \\
\! + \!CodeRL+ & 13.7 & 26.8 & 5.0 & \textbf{6.3} & \textbf{0.0} & 7.3 & 40.0 & \textbf{3.8} & \textbf{1.9} & \textbf{8.5} & 9.8 \\
\! + \!GRPO & 6.8 & 12.2 & 5.0 & 1.6 & \textbf{0.0} & 5.7 & 40.0 & 1.9 & 0.0 & 4.7 & 5.7 \\
\! + \!GiGPO & 9.1 & 13.8 & \textbf{10.0} & 1.6 & \textbf{0.0} & 5.7 & 40.0 & 1.9 & 0.0 & 2.8 & 5.9 \\ 
\rowcolor{cyan!20}
\! + \!DiDPO (ours) & \textbf{15.6} & \textbf{27.6} & \textbf{10.0} & \textbf{6.3} & \textbf{0.0} & \textbf{8.0} & \textbf{45.0} & \textbf{3.8} & \textbf{1.9} & \textbf{8.5} & \textbf{10.7} \\ \hline
Qwen3.5-4B & 31.6 & 49.6 & 26.0 & 15.9 & \textbf{0.0} & 8.8 & 50.0 & 7.5 & 0.0 & 3.8 & 14.7 \\ \hdashline 
\multicolumn{12}{l}{\textbf{\textit{Code-Reasoning Baselines}}} \\
\! + \!CoT & 40.1 & 57.7 & 38.0 & \textbf{22.0} & \textbf{0.0} & 13.2 & \textbf{70.0} & 13.2 & 0.0 & 1.9 & 18.4 \\
\! + \!CodeAct & 29.0 & 32.5 & 12.0 & 11.1 & \textbf{0.0} & 8.2 & 35.0 & 9.4 & \textbf{1.2} & 0.9 & 12.7 \\
\! + \!Self-Planning & 24.8 & 36.6 & 23.0 & 12.7 & \textbf{0.0} & 8.2 & 40.0 & 9.4 & 0.0 & 0.0 & 11.0 \\  \hdashline 
\multicolumn{12}{l}{\textbf{\textit{Coding (Agentic) RL Baselines} }}\\
\! + \!SFT & 32.7 & 50.4 & 30.0 & 13.2 & \textbf{0.0} & 4.4 & 30.0 & 1.9 & 0.0 & 0.0 & 12.4 \\
\! + \!CodeRL+ & 37.5 & 54.5 & 37.0 & 17.4 & \textbf{0.0} & 11.3 & 55.0 & 11.3 & \textbf{1.2} & 2.8 & 17.2  \\
\! + \!GRPO & 35.1 & 48.9 & 37.0 & 17.0 & \textbf{0.0} & 7.6 & 50.0 & 3.8 & 0.0 & 2.8 & 15.2 \\
\! + \!GiGPO & 36.8 & 51.2 & 40.0 & 15.9 & \textbf{0.0} & 8.2 & 55.0 & 7.5 & \textbf{1.2} & 3.8 & 16.9  \\ 
\rowcolor{cyan!20}
\! + \!DiDPO (ours) & \textbf{43.0}  & \textbf{60.9}  & \textbf{44.0}  & 20.6   & \textbf{0.0} & \textbf{13.9}  & 60.0   & \textbf{17.0}  & \textbf{1.2} & \textbf{9.4} & \textbf{22.1}  \\ \hline
\end{tabular}}\vspace{-0.4cm}
\end{table*}

\subsection{Main Results}

\noindent\textbf{Overall Performance.} Table~\ref{tab:my-table:1} reports results on two backbones. DiDPO achieves the highest average on both: 48.4\% with Qwen2.5-Coder-7B and 58.6\% with Qwen3.5-4B, surpassing the strongest baseline GiGPO by 4.2 and 4.9 respectively. DiDPO narrows the gap to GPT-5.5 from 56.4\% (base model 35.3\% vs.\ 91.7\%) to 43.3\% (48.4\% vs.\ 91.7\%). Table~\ref{tab:my-table:2} further reports results on competition benchmarks, where DiDPO achieves 15.6\% on USACO with the 7B backbone, more than doubling GRPO's 6.8\%.

\noindent\textbf{Compared with Code-Reasoning Baselines.} Among reasoning baselines, no single strategy dominates: CoT leads on APPS while Self-Planning is strongest on LiveCodeBench. CodeAct underperforms even the base model on APPS for the 7B backbone, i.e., 14.7\% vs.\ 16.9\%, showing that untrained tool-use can disrupt synthesis. DiDPO outperforms the best reasoning method by 5.4  on APPS and 12.8 on LiveCodeBench with the 7B backbone, confirming that RL provides gains inaccessible to prompting alone.

\noindent\textbf{Compared with Coding / Agentic RL Baselines.} DiDPO improves over GRPO by 5.6 on average with the 7B backbone. Since both share the same episode-level advantage, these gains are directly attributable to sub-diff credit. Compared with GiGPO, DiDPO leads by 4.2. This is because GiGPO groups by identical environment states, which fails when functionally analogous edits modify different code regions; the advantage is most pronounced on APPS Interview, i.e., +10.4, where multi-step reasoning spans multiple functions. SFT slightly outperforms base models but largely underperforms all RL methods applied over it. This suggests that, building on format alignment capability, RL methods can further enhance the model's performance in long-horizon competitive coding and reasoning.

\begin{figure}[t] 
    \centering
    \includegraphics[width=1\linewidth]{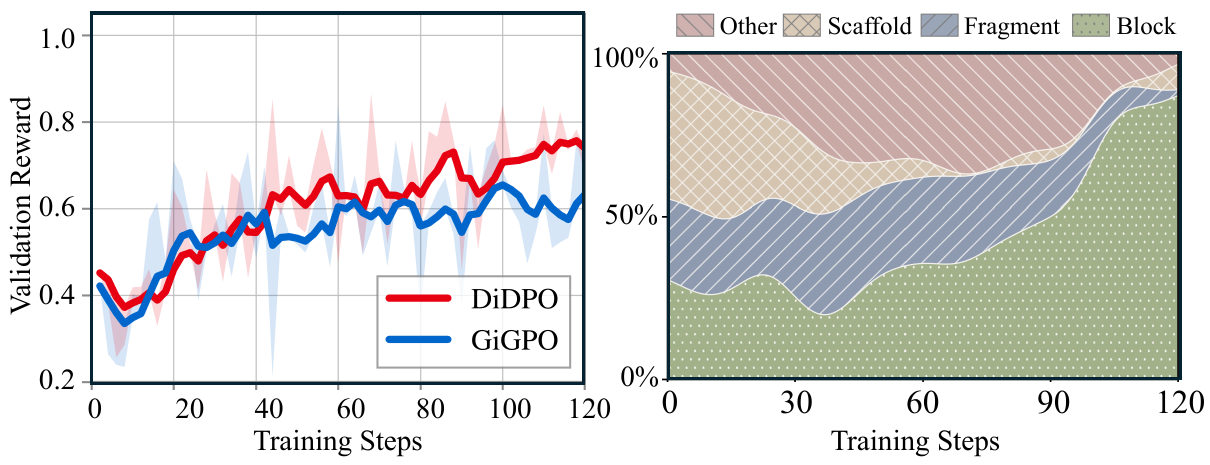} 
    \vspace{-0.6cm}
    \caption{\textbf{Left:} validation reward over training. \textbf{Right:} Evolution dynamics of group types in DiDPO, which are classified by GPT-5.5. \textit{Block:} sub-diffs on meaningful functional blocks. \textit{Fragment:} sub-diffs on fragments. \textit{Scaffold:} sub-diffs on object definitions and declarations. \textit{Other:} sub-diffs on meaningless lines (like blank lines).} 
    \vspace{-0.5cm}
    \label{fig:dynamics}
\end{figure}

\subsection{Analysis}

\noindent\textbf{Learning Dynamics Analysis.} Figure~\ref{fig:dynamics} (left) shows APPS accuracy over training. All methods improve similarly in the first 20 steps, driven by the shared episode-level advantage. After step 40, DiDPO continues to climb while GiGPO plateaus, indicating that sub-diff credit becomes more informative as the policy diversifies its edits. Figure~\ref{fig:dynamics} (right) tracks group composition with 4 types classified by GPT-5.5. \textit{Block} (function bodies and loop structures) steadily increase their share, while \textit{Fragment} groups (short isolated edit segments), and \textit{Scaffold} groups (imports and definitions) decline. This shift matches the DiDPO design, where functional blocks earn higher GS and dominate anchor selection, thus providing more meaningful credit. The evolution of group mass distributions is provided in  \textbf{\underline{Appendix ~\ref{appx:f}}}.

\begin{table}[t]
\caption{Ablation studies on different design of ``Groupability Score (GS)'' \textbf{(left)} and components \textbf{(right)}.}
\vspace{-0.2cm}
\begin{center}
\renewcommand{\arraystretch}{1}
\setlength\tabcolsep{2pt}
\resizebox{0.48\textwidth}{!}{
\begin{minipage}{0.32\textwidth}
\centering
\begin{tabular}{l|cc}
\hline
Design & APPS & OJBench \\ \hline
$\bar L(\bm{c})\times n(\bm{c})$ & 16.5 & 1.9 \\
$\bar L(\bm{c})+ n(\bm{c})$ & 24.4 & 5.7 \\
Qwen3.6-27B Judge & 21.5 & 1.9 \\
\rowcolor{cyan!20}
Ours (Eq. \ref{eq:anchor-groupability}) & 31.3 & 8.0 \\ \hline
\end{tabular}
\end{minipage}
\hspace{0.02\textwidth}
\begin{minipage}{0.32\textwidth}
\centering
\begin{tabular}{l|cc}
\hline
Method & APPS  & OJBench \\ \hline
\rowcolor{cyan!20}
DiDPO & 31.3 & 8.0 \\
w/o $A^E$ & 10.4 & 3.8 \\
w/o $A^D$ & 23.8 & 5.7 \\
w/o sub-diff & 25.0 & 4.4 \\ \hline
\end{tabular}
\end{minipage}}
\end{center}
\label{tab:3333}
\vspace{-0.6cm}
\end{table}

\noindent\textbf{Different Groupability Score (GS) Design.} Table~\ref{tab:3333} (left) compares GS formulations. Our design in Eq.~\ref{eq:anchor-groupability} achieves 31.3\% on APPS, versus 24.4\% for the additive variant and 21.5\% for an LLM judge (which uses an LLM to group sub-diffs) based on Qwen3.6-27B. We think the additive form allows one strong factor to compensate for a weak one, but the saturating exponential compresses both factors into $(0,1)$, preventing either from dominating while keeping moderate pairs viable. The costly LLM judge also underperforms because it favors a few large-sized groups, while format issues also render it ineffective occasionally.

\noindent\textbf{Ablation Study of DiDPO's Components.} We ablate DiDPO's components and show the results in Table~\ref{tab:3333} (right). Removing $A^E$ drops APPS from 31.3\% to 10.4\%, confirming episode-level signal is indispensable. Removing $A^D$ reduces to 23.8\%, which is comparable to GRPO and isolates the 7.5-point marginal gain from localized credit. Removing sub-diff decomposition, i.e., treating each whole diff as the atomic grouping unit, yields 25.0\%, outperforming GRPO but 6.3-point below full DiDPO. This gap  together with Figure~\ref{fig:show}  confirms that: without sub-diff decomposition, partially overlapping diffs cannot be matched, producing smaller and noisier groups. All three components are necessary and their combination exceeds additive contribution.

\begin{figure}[t] 
    \centering
    \includegraphics[width=1\linewidth]{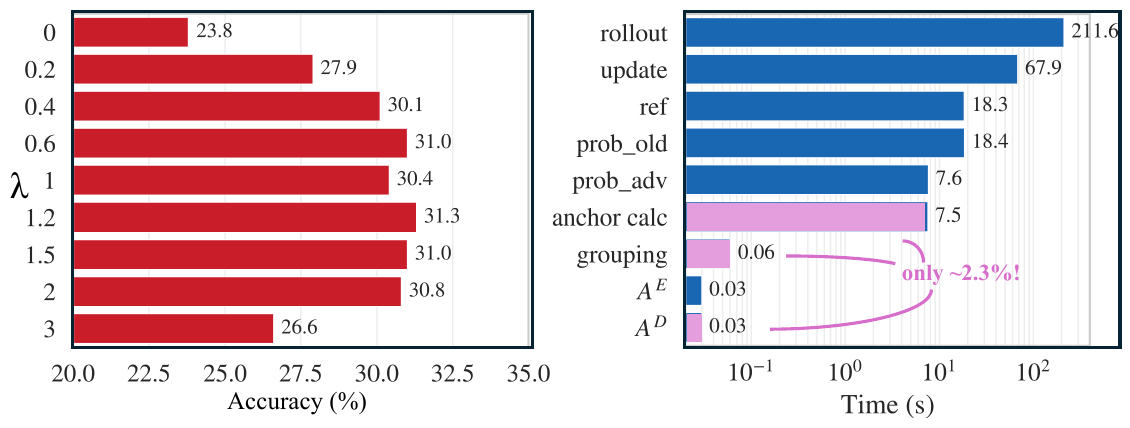} 
    \vspace{-0.7cm}
    \caption{\textbf{Left:} Sensitivity analysis on $\lambda$. \textbf{Right:} Analysis of per-step training time breakdown. Pink bars are DiDPO-specific. Blue bars are these shared with GRPO.} 
    \vspace{-0.5cm}
    \label{fig:sensei}
\end{figure}

\noindent\textbf{Sensitivity Analysis of $\lambda$.} As shown in Figure~\ref{fig:sensei} (left), the performance with respect to $\lambda$ follows an inverted U-shape, peaking near $\lambda=1.2$. Below 0.6, the diff-level signal is suppressed and DiDPO reverts to GRPO-level performance, consistent with the $A^D$ ablation. Above 1.2, the policy overfits to local credit, learning edits that resemble high-reward peers but fail to compose globally. The method tolerates moderate variation around the optimum: accuracy is stable across 0.6 to 1.2, and we select $\lambda$ as 1.2 in practice.

\noindent\textbf{Scaling and Efficiency Analysis.} Figure~\ref{fig:sensei} (right) breaks down per-step training time. DiDPO  only  adds approximately 2.3\% overhead over GRPO, dominated by cross-rollout similarity computation for anchor calculation (Eq.~\ref{eq:anchor-match}). This cost scales quadratically with sub-diff candidates but is bounded in practice by the similarity threshold $\eta$. Greedy anchor selection (Eq.~\ref{eq:anchor-selection}) is linear and negligible. Moreover, this overhead is training-only since inference uses standard autoregressive decoding with no grouping. Given the 5.6 average gain over GRPO (Table~\ref{tab:my-table:1}), the overhead is well justified, and caching or approximate search could further reduce it.


\section{Conclusion}

DiDPO starts from the observation that comparable behavior in coding agent trajectories often appears as recurring or partially recurring diffs across rollouts. By selecting similar sub-diffs from these diffs with the groupability score and constructing dynamic advantage groups, DiDPO turns generated code into the substrate of critic-free policy optimization. Extensive experiments show that DiDPO mostly outperforms existing baselines. DiDPO is effective and a practical step toward coding agentic RL.

\bibliographystyle{plainnat}
\bibliography{main}


\clearpage
\newpage
\appendix

\section{Proofs for Theoretical Foundations}\label{appx:a}

\subsection{Proof of Theorem 1}
\begin{proof}
Fix one prompt and consider a DiDPO group $\mathcal{G}$ induced by one selected anchor. By assumption, every sub-diff $s\in\mathcal{G}$ is within correspondence distance $\epsilon$ of the anchor prototype $a$, namely $\Delta(s,a)\le\epsilon$. By the triangle inequality for the induced correspondence distance, any two sub-diffs $s,u\in\mathcal{G}$ satisfy
\[
\Delta(s,u)\le \Delta(s,a)+\Delta(a,u)\le 2\epsilon .
\]
Since the local reward contribution $r(\cdot)$ is $L$-Lipschitz with respect to $\Delta$,
\[
|r(s)-r(u)|\le L\Delta(s,u)\le 2L\epsilon .
\]
Let the ideal local contrast for $s$ be
\[
A^\star(s)=r(s)-\frac{1}{|\mathcal{G}|}\sum_{u\in\mathcal{G}}r(u).
\]
If exact matches were available, the group baseline would compare only identical causal code changes. Anchor-based matching replaces exact equality by $\epsilon$-accurate correspondences. The above Lipschitz bound implies that every local reward term in the anchor-based contrast differs from its exact-match counterpart by at most a constant multiple of $L\epsilon$. Averaging over $\mathcal{G}$ preserves the same order, hence
\[
\left|
\mathbb{E}\big[\hat A^{D}(s)\big]-A^\star(s)
\right|
\le C L\epsilon ,
\]
for a constant $C$ independent of rollout length and group size. This gives the stated $O(L\epsilon)$ bound.
\end{proof}

\subsection{Proof of Theorem 2}
\begin{proof}
For rollout $i$, write the step return as
\[
R_i=r(s_i)+\xi_i,
\]
where $s_i$ is the matched sub-diff and $\xi_i$ collects non-causal trajectory effects, including copied context, irrelevant edits, exploration commands, and formatting behavior. Assume $\mathbb{E}[\xi_i]=0$ and $\mathrm{Var}(\xi_i)=\sigma_\xi^2$.

Consider a DiDPO group $\mathcal{G}$ of size $m>1$ whose matched sub-diffs have correspondence error at most $\epsilon$. By Theorem 1, replacing the ideal local contrast with an anchor-based sub-diff contrast introduces squared bias
\[
O(L^2\epsilon^2).
\]
It remains to bound the residual variance. The group baseline averages the returns of $m$ matched sub-diffs. Since the non-causal terms are zero mean and are averaged within the matched group, their contribution scales as
\[
\mathrm{Var}\left(\frac{1}{m}\sum_{s_j\in\mathcal{G}}\xi_j\right)
=O(\sigma_\xi^2/m),
\]
with the same order under bounded weak correlation among matched rollouts. Combining the squared bias and the residual variance yields
\[
\mathrm{MSE}(\hat A_i^{D})
\le
O(L^2\epsilon^2)+O(\sigma_\xi^2/m).
\]

For episode-level broadcasting, the same scalar return $R_i$ is assigned to all response tokens. Tokens inside the matched sub-diff therefore receive a signal that still contains $\xi_i$. Since this estimator does not condition on the sub-diff group, it has no averaging mechanism that removes non-causal variation from unrelated parts of the trajectory. Its error consequently retains an $O(\sigma_\xi^2)$ contamination term. Thus, when reliable matched sub-diffs exist and $m>1$, the DiDPO local estimator has a lower credit-error order for tokens in the causal code region.
\end{proof}

\section{Greedy Selection for Anchors}\label{appx:b}

Equation 10 selects a small set of anchors that covers as many useful matched sub-diffs as possible without repeatedly counting the same occurrence. Let $\mathcal{I}$ denote all candidate sub-diffs and let $\mathcal{C}$ denote all anchor candidates. For a selected anchor set $\mathcal{A}$, define
\[
F(\mathcal{A})=
\sum_{s\in\mathcal{I}}\max_{a\in\mathcal{A}}
\mathbf{1}[s\in\mathcal{O}(a)]\mathrm{GS}(a).
\]
This objective is monotone because adding an anchor cannot reduce the best covered score of any sub-diff. It is also submodular. To see this, fix one candidate sub-diff $s$ and define
\[
F_s(\mathcal{A})=\max_{a\in\mathcal{A}}
\mathbf{1}[s\in\mathcal{O}(a)]\mathrm{GS}(a).
\]
If $\mathcal{A}\subseteq\mathcal{B}$, then $\mathcal{B}$ already covers $s$ at least as well as $\mathcal{A}$. Adding a new anchor $a$ can therefore improve $F_s(\mathcal{B})$ no more than it improves $F_s(\mathcal{A})$:
\[
F_s(\mathcal{A}\cup\{a\})-F_s(\mathcal{A})
\ge
F_s(\mathcal{B}\cup\{a\})-F_s(\mathcal{B}).
\]
Summing over all $s\in\mathcal{I}$ preserves this diminishing-return property, so $F$ is a monotone submodular function. The exact maximizer under $|\mathcal{A}|\le K$ requires a combinatorial search over anchor subsets, which is unnecessary here because the objective belongs to the standard cardinality-constrained submodular maximization family. DiDPO therefore uses greedy selection, which is the canonical approximation method for this setting.

\begin{algorithm}[h]
\caption{Greedy Anchor Selection}
\begin{algorithmic}[1]
\STATE Initialize $\mathcal{A}\leftarrow\emptyset$.
\WHILE{$|\mathcal{A}|<K$}
\STATE Choose $a^\star=\arg\max_{a\in\mathcal{C}\setminus\mathcal{A}}
\big(F(\mathcal{A}\cup\{a\})-F(\mathcal{A})\big)$.
\IF{$F(\mathcal{A}\cup\{a^\star\})=F(\mathcal{A})$}
\STATE Break.
\ENDIF
\STATE Update $\mathcal{A}\leftarrow\mathcal{A}\cup\{a^\star\}$.
\ENDWHILE
\STATE Return $\mathcal{A}$.
\end{algorithmic}
\end{algorithm}

By the classical result of Nemhauser et al.~\cite{nemhauser1978analysis}, this greedy procedure obtains a $(1-1/e)$ approximation to the best anchor set under the budget $K$. Thus the greedy rule is not merely a heuristic: for the objective in Eq. 10, repeatedly choosing the largest marginal gain is the standard guaranteed solution strategy. In practice, DiDPO computes marginal gains using the currently best covered score of each candidate sub-diff, so each greedy step only updates sub-diffs covered by the newly added anchor.

\section{Evaluation Benchmarks}\label{appx:c}

We evaluate DiDPO on eight coding benchmarks, ranging from compact Python functions to competition-level programming. All results are based on execution correctness, which requires the code generated by agents to pass all of the test cases. When a benchmark provides meaningful difficulty partitions, we report both the overall score and the corresponding subcategory scores.

\paragraph{APPS.}
APPS~\cite{apps2021hendrycks} contains 10,000 programming problems collected from coding challenge platforms. Its difficulty split is central to our evaluation: introductory problems mostly test routine implementation, interview problems require more deliberate algorithm design, and competition problems stress longer reasoning. We use 500 APPS examples in Table~1, including 120 introductory, 250 interview, and 130 competition problems.

\paragraph{HumanEval.}
HumanEval~\cite{chen2021evaluating} contains 164 hand-written Python function synthesis tasks with unit tests. Since its problems are compact, it mainly checks whether DiDPO preserves ordinary function-level generation while improving longer coding trajectories.

\paragraph{MBPP.}
MBPP~\cite{austin2021program} consists of crowd-sourced Python tasks intended for entry-level programmers. It covers basic programming patterns and standard-library usage more broadly than HumanEval. We report the 399-example evaluation subset shown in Table~1.

\paragraph{LiveCodeBench.}
LiveCodeBench~\cite{jain2025livecodebench} continuously collects recent problems from platforms such as LeetCode, AtCoder, and CodeForces. This temporal design makes it useful for testing contamination-resistant code generation. We use its code-generation setting with 150 processed examples.

\paragraph{LeetCode.}
We use LeetCode problems from LeetCodeDataset~\cite{xia2025leetcodedataset}. The dataset curates Python problems with rich metadata, extensive tests, and temporal splits. Our evaluation includes 112 examples, which mainly examine self-contained data-structure and algorithmic reasoning.

\paragraph{USACO.}
USACO~\cite{shi2024can} evaluates problems from USA Computing Olympiad contests. We use the full 307-problem benchmark and report its official tier structure: 123 Bronze, 100 Silver, 63 Gold, and 21 Platinum problems. The tiered results are important because higher levels increasingly require non-obvious algorithms and careful implementation.

\paragraph{OJBench.}
OJBench~\cite{wang2025ojbench} dataset targets competition-level code reasoning with strict online-judge evaluation. We use 159 examples, split into 20 Easy, 53 Medium, and 86 Hard problems.

\paragraph{ICPC.}
ICPC-Eval~\cite{xu2026icpc} focuses on problems selected from ICPC contests. These tasks are typically less template-like than short interview questions and often require both algorithm selection and implementation discipline. We evaluate on 106 ICPC instances.

\begin{table*}[h]
\centering
\caption{Evaluation datasets used in DiDPO experiments. ``Reported subsets'' lists the partitions shown in the main tables.}
\setlength\tabcolsep{16pt} 
\resizebox{2\columnwidth}{!}{
\begin{tabular}{l|l|c|l}
\hline
Benchmark & Type & Samples & Reported subsets \\ \hline
APPS & Mixed programming challenge & 500 & All / Intro / Interview / Competition \\
HumanEval & Function-level synthesis & 164 & All \\
MBPP & Basic Python synthesis & 399 & All \\
LiveCodeBench & Temporal contest code generation & 150 & All \\
LeetCode & Interview-style algorithmic reasoning & 112 & All \\
USACO & Olympiad programming & 307 & All / Bronze / Silver / Gold / Platinum \\
OJBench & Online-judge competition programming & 159 & All / Easy / Medium / Hard \\
ICPC & ICPC-style competitive programming & 106 & All \\ \hline
\end{tabular}}
\end{table*}

\begin{figure*}[t] 
    \centering
    \includegraphics[width=1\linewidth]{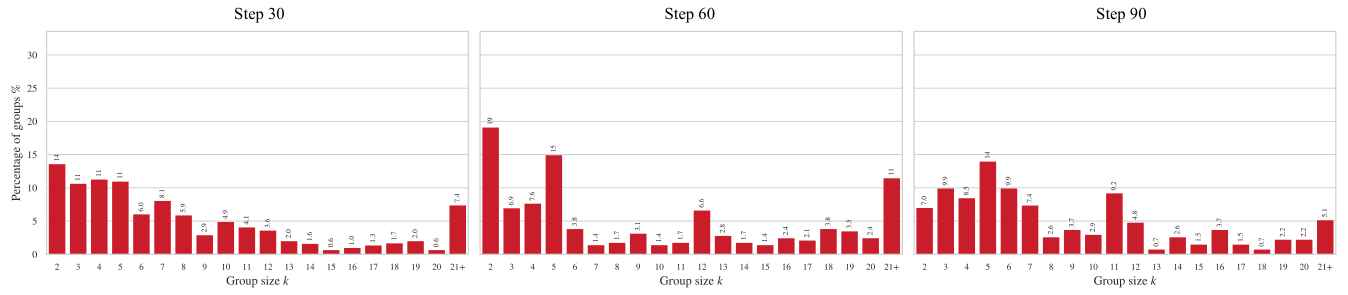} 
    \vspace{-0.7cm}
    \caption{Evolution of group mass over training. We show the evolution dynamics at step 30, 60 and 90 respectively.} 
    \vspace{-0.5cm}
    \label{fig:mass-evolve}
\end{figure*}

\section{Baselines}\label{appx:d}

We compare DiDPO with three groups of baselines. Large models provide reference performance. Prompting and agent-format methods test whether better inference alone is sufficient. RL baselines are used to isolate the effect of our sub-diff credit assignment.

\paragraph{Large model references.}
Kimi-K2.6 is a frontier coding-oriented model from the Kimi series. Qwen3.6-27B is a dense coding model from the Qwen family, with strong instruction following and code-generation ability. GLM-5.2 and GPT-5.5 are general-purpose large models with strong reasoning and programming performance. Skywork-OR1~\cite{he2025skywork} is an open reasoning model trained with reinforcement learning. We report these models as reference systems and do not further train them in our pipeline.

\paragraph{Base backbones and SFT.}
Qwen2.5-Coder-7B~\cite{hui2024qwen25coder} is a code-specialized open model trained on large-scale code data and code-related instruction data. Qwen3.5-4B is a smaller Qwen backbone used to test whether DiDPO remains effective under a more compact model scale. For each backbone, we report the raw model and the SFT variant. The SFT model is trained on trajectories written in the same thought-action format used by DiDPO, so it reflects format alignment before RL optimization.

\paragraph{Code-reasoning baselines.}
Chain-of-Thought (CoT)~\cite{wei2022chain} lets the model reason before writing code. CodeAct~\cite{wang2024codeact} represents actions as executable code and feeds execution observations back to the agent. Self-Planning~\cite{jiang2024self} first plans the solution and then implements it step by step. Together, these baselines test the foundational capabilities of LLMs with training-free harness.

\paragraph{Coding and agentic RL baselines.}
CodeRL+~\cite{jiang2026coderl+} adds execution-semantics alignment to RLVR for code generation. GRPO~\cite{shao2024deepseekmath,guo2025deepseekr1} is the closest episode-level baseline because it also uses group-relative advantages without a critic. GiGPO~\cite{feng2025gigpo} is the strongest agentic credit-assignment baseline, since it builds step-level groups from recurring environment states. DiDPO differs by forming groups inside code diffs, where one coding action may contain several functional sub-diffs.

\paragraph{Controlled comparison.}
For trainable baselines, we keep the task distribution, executable reward, thought-action format, and evaluation protocol aligned with DiDPO whenever applicable. The comparison with GRPO measures the gain from adding sub-diff credit to episode-level advantages. The comparison with GiGPO tests whether state-level grouping is sufficient once the action becomes a structured code diff.

\section{Hyperparameter Settings}\label{appx:e}

The main DiDPO run uses Qwen2.5-Coder-7B initialized from the SFT checkpoint. We train on the multi-turn coding environment with rollout group size 32. Each rollout can take at most eight interaction steps, and each edit is followed by executable feedback. This setting keeps the task long enough to expose multi-step code revision behavior while keeping the training budget controlled.

For DiDPO-specific grouping, we set the diff-level advantage weight to $\lambda=1.2$. We also use a soft matching for sub-diffs, where the sub-diff similarity threshold is $\eta=0.8$, which avoids grouping loosely related edits. The groupability score uses the saturating form in Eq.~10; in implementation we add an extra length scale $s_0=8.0$ and support scale $g_0=8.0$. For tokens covered by multiple candidate groups, we select the group with the largest groupability score. We allow at most 64 anchors per task group.

For the shared policy optimization setting, we use discount factor $\gamma=0.95$, learning rate $1\times10^{-6}$, PPO clipping ratio 0.2, KL coefficient 0.01, and entropy coefficient 0.001. PPO is run for one epoch per batch with mini-batch size 48. Invalid actions receive a small penalty with coefficient 0.01. The training rollout temperature is 1.0, while validation uses temperature 0.6 with sampling enabled. Validation accuracy is averaged over 5 runs.

The maximum prompt length is 8192 and the maximum response length is 4096. We train on one node with 8@H20 GPUs, tensor parallel size 2, and 120 total training steps. Checkpoints are saved every 20 steps, and the main run keeps at most three actor checkpoints.

\section{Evolution of Group Mass Distributions}\label{appx:f}
Figure~\ref{fig:mass-evolve} shows how group mass distributes across training steps. At Step~30, the distribution concentrates at small group sizes, dominated by short fragment-level matches typical of early policy outputs. By Step~60, medium and large groups emerge as the policy produces more diverse edits and the groupability score promotes anchors with broader semantic scope. At Step~90, the distribution becomes markedly more uniform across group sizes, indicating that DiDPO forms credit groups at multiple granularities rather than collapsing to a single scale.

This diversification follows from the design of the groupability score (Eq. 9). Early in training, homogeneous diffs restrict matching to small, highly similar sub-diffs. As the policy explores more varied edits, the groupability score selects anchors that balance semantic scope with group mass, yielding groups at multiple scales. The resulting uniformity connects to the bias-variance decomposition in Section~3.5: larger groups reduce variance through cross-rollout averaging ($O(\sigma_\xi^2/m)$), while well-aligned smaller groups supply low-bias credit signals ($O(L\epsilon)$). The coexistence of both types at convergence explains DiDPO's sustained improvement beyond the plateau of single-scale credit assignment methods.


\end{document}